\relax
\documentclass[letterpaper]{article} 
\usepackage{arxiv}  
\usepackage{times}  
\usepackage{helvet}  
\usepackage{courier}  
\usepackage[hyphens]{url}  
\usepackage{graphicx} 
\usepackage{natbib}  
\usepackage{caption} 
\DeclareCaptionStyle{ruled}{labelfont=normalfont,labelsep=colon,strut=off} 
\usepackage{algorithm}

\usepackage{newfloat}
\usepackage{listings}
\floatstyle{ruled}
\newfloat{listing}{tb}{lst}{}
\floatname{listing}{Listing}
\usepackage{times}  
\usepackage{helvet} 
\usepackage{courier}  
\usepackage[hyphens]{url}  
\usepackage{graphicx} 

\usepackage{hyperref}
\usepackage{caption}
\usepackage{graphicx}
\usepackage{subfig}
\usepackage{tikz}
\usetikzlibrary{tikzmark, shapes.arrows,arrows,automata,positioning}
\usepackage{url}
\usepackage{array}
\usepackage{booktabs}

\usepackage{algorithm}
\usepackage{multirow}
\usepackage[]{algpseudocode}
\usepackage{amsmath,amsthm,amssymb}

\usepackage{thmtools,thm-restate}

\newtheorem{innercustomthm}{Theorem}
\newenvironment{mythm}[1]
  {\renewcommand\theinnercustomthm{#1}\innercustomthm}
  {\endinnercustomthm}

\newtheorem{definition}{Definition}

\newtheorem{assumption}{Assumption}

\newtheorem{theorem}{Theorem}
\newtheorem{prop}{Proposition}

\newcommand{\ie}{\textit{i}.\textit{e}., }

\DeclareMathOperator*{\argmax}{arg\!\max}  
\usepackage{natbib}  
\usepackage{caption} 
\title{Enhancing Counterfactual Classification via Self-Training}

\author {
    Ruijiang Gao\footnote{This work was done while the author was an intern at IBM Research},\textsuperscript{\rm 1}
    Max Biggs, \textsuperscript{\rm 2}
    Wei Sun, \textsuperscript{\rm 3} 
    Ligong Han \textsuperscript{\rm 4} \\
}
\affiliations {
    \textsuperscript{\rm 1} University of Texas at Austin \\
    \textsuperscript{\rm 2} University of Virginia \\
    \textsuperscript{\rm 3} IBM Research \\
    \textsuperscript{\rm 4} Rutgers University \\
    ruijiang@utexas.edu, mbiggs@darden.virginia.edu, sunw@us.ibm.com, lh599@scarletmail.rutgers.edu
}

\begin{document}

\maketitle

\begin{abstract}
Unlike traditional supervised learning, in many settings only partial feedback is available. We may only observe outcomes for the chosen actions, but not the counterfactual outcomes associated with other alternatives. Such settings encompass a wide variety of applications including pricing, online marketing and precision medicine. A key challenge is that observational data are influenced by historical 
policies deployed in the system, yielding a biased data distribution. We approach this task as a domain adaptation problem and propose a self-training algorithm which imputes outcomes with 
\textit{categorical} 
values for \textit{finite} unseen actions in the observational data to simulate a randomized trial through pseudolabeling, which we refer to as Counterfactual Self-Training (CST). CST iteratively imputes pseudolabels and retrains the model. In addition, we show input consistency loss can further improve CST performance which is shown in recent theoretical analysis of pseudolabeling. We demonstrate the effectiveness of the proposed algorithms on both synthetic and real datasets.

\end{abstract}

\section{Introduction}

Counterfactual inference~\citep{pearl2000models} attempts to address a question central to many applications - \textit{What would be the outcome had an alternative action was chosen?} This includes selecting relevant ads to engage with users in online marketing~\citep{li2010contextual}, determining prices that maximize profit in revenue management~\citep{bertsimas2016power}, or designing the most effective personalized treatment for a patient in precision medicine~\citep{xu2016bayesian}. With observational data, we have access to past actions, their  outcomes, and possibly some context, but in many cases not the complete knowledge of the historical policy which gave rise to the action~\citep{shalit2017estimating}.  Consider a pricing setting in the form of targeted promotions. We might record information about a customer (context), the promotion offered (action) and whether an item was purchased (outcome), but we do not know why a particular promotion was selected. 

Unlike traditional supervised learning with full feedback, we only observe partial feedback for the chosen action in observational data,  but not the outcomes associated with other alternatives (i.e., in the pricing example, we do not observe what would occur if a different promotion was offered). In contrast to the gold standard of a randomized controlled trial,  observational data are influenced by the historical policy deployed in the system which may over or underrepresent certain actions,  yielding a biased data distribution. A naive but widely used approach is to learn a machine learning algorithm directly from observational data and use it for prediction. This is often referred to as the direct method (DM)~\citep{dudik2014doubly}. Failure to account for the bias introduced by the historical policy often results in an algorithm which has high accuracy on the data it was trained on, but performs considerably worse under different policies~\citep{shalit2017estimating}. 
For example, in the pricing setting, if historically most customers who received high promotion offers have a certain profile, then a model based on direct method may fail to produce reliable predictions on these customers when low offers are given. 
In this paper, we focus on \textbf{\textit{Counterfactual Classification (CC)}}, where the outcome of each action is categorical. For instance, one may want to estimate whether a customer will purchase an item in a pricing application or a user  clicks on an ad in online marketing. This setting is in contrast to the body of  literature on treatment effect estimation which assumes a continuous outcome~\citep{yoon2018ganite,shalit2017estimating,shi2019adapting,alaa2017bayesian}. 

To overcome the limitations of the direct method, \citet{shalit2017estimating,johansson2016learning,lopez2020cost} cast counterfactual learning as a {\em domain adaptation} problem, where the source domain is observational data and the target domain is a {\em randomized trial} whose assignment of actions follows a uniform distribution for a given context. 
The key idea is to map contextual features to an embedding space and jointly learn a representation  
that encourages similarity between these two domains, leading to better counterfactual inference. 
The embedding is generally learned by a neural network and the estimation of the domain gap is usually slow to compute. Such representation learning approaches also pose challenges for interpretable models since the learned representation often loses the semantic meaning of origin features. 
In this paper, while we also view counterfactual classification as a domain adaptation problem between observational data and an ideal randomized trial, we take a different approach - instead of estimating the domain gap between the two distributions via an embedding, we explicitly simulate a randomized trial by imputing pseudolabels for the unobserved actions in the observational data. The optimization process is done by iteratively updating the pseudolabels and a model that is trained on both the factual and the counterfactual data, as illustrated in Figure~\ref{fig:st_ex}. As this method works in a self-supervised fashion~\citep{zou2018domain,amini2002semi}, 
we refer to our proposed framework as Counterfactual Self-Training (CST).

The contributions of our paper are as follows. First, we propose a novel pseudolabeling self-training algorithm for counterfactual inference. 
To the best of our knowledge, this is the first application of a self-training algorithm using pseudolabeling for the counterfactual classification problem. In contrast to the existing methods from domain adaption on counterfactual inference, CST shares the advantage of PL methods which are data-modality-agnostic and work for a wide range of models. Second, we show incorporating a counterfactual input consistency loss can further improve CST's performance, which is consistent with recent theoretical analysis of pseudolabeling when the underlying data distribution and pseudolabeler satisfies certain assumptions such as sufficiently high labeling accuracy and the underlying data is well-separated so that the wrong pseudolabels can be denoised by close correct labels.
This finding suggests promising results for pseudolabeling based approaches in counterfactual classification problems.
Third, we present comprehensive experiments on toy data, several synthetic datasets and real datasets converted from multi-label classification tasks to evaluate our method against state-of-the-art baselines. In all experiments, CST shows competitive or superior performance against all the baselines. 
\begin{figure*}[t!]
    \centering
    \includegraphics[width=\linewidth]{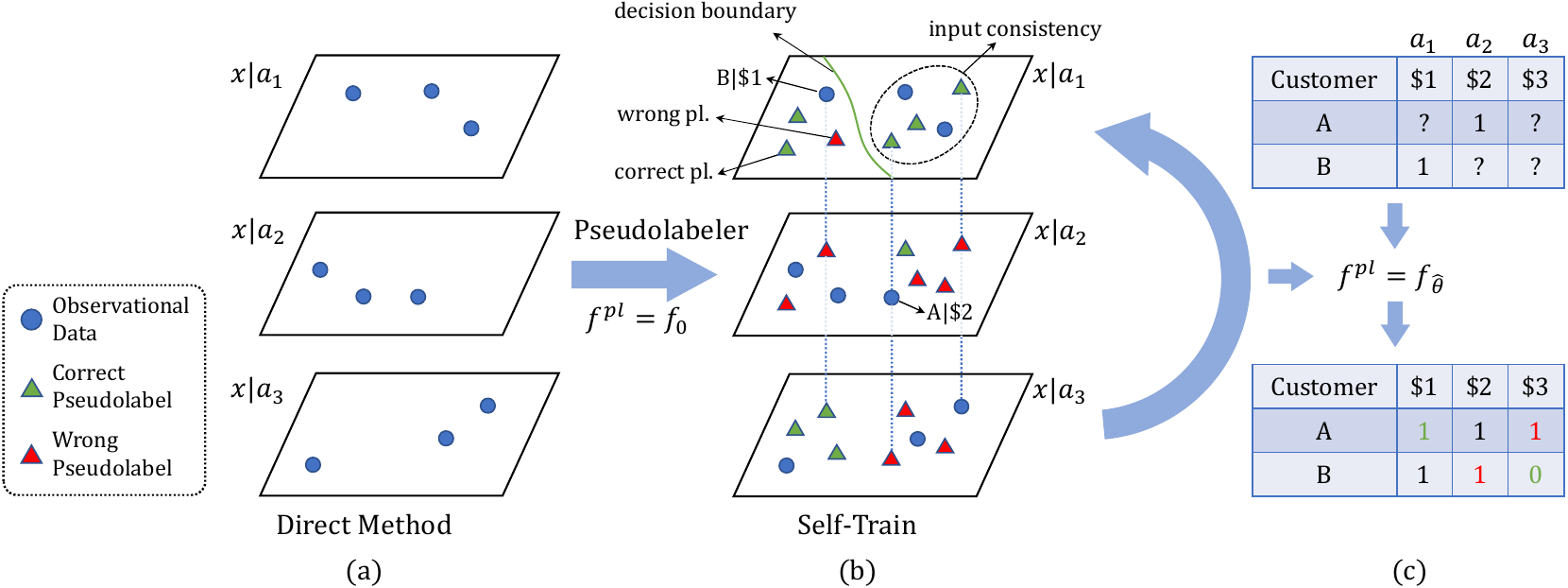}
    \caption{Illustration of the proposed Counterfactual Self-Training (CST) framework. 
    There are two sales records (observational data) shown in the table, \ie  Customer A was offered \$2 and bought an item; Customer B was offered \$1 and also bought. The question marks in the tables represent the counterfactual outcomes which we do not observe. For all these unseen counterfactual outcomes, pseudolabels which are colored in the tables are imputed by a model and are used to augment the observational data. The model is subsequently updated by training on both the imputed counterfactual data and the factual data. This iterative training procedure continues until it converges. Blue points represent factual observations. Triangles represent imputed pseudolabels. For pseudolabels, red means wrong labels and green means correct labels. The input consistency regularization correct the wrong pseudolabel with its neighbors' correct labels, resulting in a refined decision boundary.}
    \label{fig:st_ex}
\end{figure*}

\section{Related Work}

Counterfactual policy optimization and evaluation has received a lot of attention in the machine learning community in recent years~\citep{swaminathan2015counterfactual,joachims2018deep, shalit2017estimating,lopez2020cost,kallus2019classifying,kallus2018confounding,Wang2019BatchLF,gao2021human,biggs2021loss}. Most of the proposed algorithms can be divided into two categories:  counterfactual risk minimization (CRM) and direct method (DM). Both can be used together to construct  doubly robust estimators~\citep{dudik2014doubly} to further improve efficiency. 
CRM, also known as off-policy learning or batch learning from bandit feedback (BLBF), typically utilizes inverse propensity weighting (IPW)~\citep{rosenbaum1987model, rosenbaum1983central} to account for the bias in the data. ~\citet{swaminathan2015counterfactual} introduces the CRM principle with a variance regularization term derived from an empirical Bernstein bound~\citep{maurer2009empirical} for finite samples. In order to reduce the variance of the IPW estimator, ~\citet{swaminathan2015self} proposes a self-normalized estimator, while BanditNet~\citep{joachims2018deep}
utilizes the baseline technique~\citep{greensmith2004variance} in deep nets. As pointed out by~\citet{lefortier2016large},  CRM-based methods tend to struggle with medium to large action spaces in practice. Moreover, CRM-based methods generally require a known (or accurately estimated) and stochastic logging policy, along with full support on the action space.  When either one of the requirements is violated, ~\citet{sachdeva2020off, kang2007demystifying} observe that the direct method often demonstrates a more robust performance. In CRM, a continuous reward is often assumed and the objective is to either maximize policy reward or evaluate the policy reward accurately. In the counterfactual classification problem, such rewards are not well-defined and may be later defined differently by various possible downstream tasks. 

Another line of research in learning from observational data is often referred to as estimating Individualized Treatment Effects (ITE)~\citep{shpitser2012identification,forney2019counterfactual} or conditional average treatment effect (CATE), which is defined as the difference of expected outcomes between two actions, with respect to a given context. The main challenge of identifying ITE is that unlike an ideal randomized trial, observational data is biased and we do not have the access to the counterfactuals.~\citet{hill2011bayesian} uses a bayesian nonparametric algorithm to address this issue.~\citet{yoon2018ganite} propose  using generative adversarial nets to capture the uncertainty in the counterfactual distributions to facilitate ITE estimation.~\citet{johansson2016learning, shalit2017estimating} approach counterfactual inference with representation learning and domain adaptation. 
~\citet{lopez2020cost} further extends this framework to multiple treatments using Hilbert-Schmidt Independence Criterion (HSIC)~\citep{gretton2008kernel} and achieves state-of-the-art performance. The HSIC proposed in~\citet{lopez2020cost} has a computation time of at least $\mathcal{O}(N^2)$, making its training process slow. While the aforementioned methods and our approach can be viewed as extensions to the direct method, similarly to CRM methods, most of them require a specific reward function a priori and are either designed for binary treatment or cannot directly be applied in our Counterfactual Classification problem which treats outcome as a categorical variable. 
We tackle the domain adaptation problem differently by explicitly augmenting the observational data to create a simulated randomized trial via self-training. Some works develop meta-learners classified as X-, T-, S-learners~\cite{kunzel2019metalearners}, for example,~\citet{hill2011bayesian} is an instance of S-learner. Our approach is similar to X-learner in the sense that both methods use pseudolabels to generate counterfactuals. However, X-learner only considers binary actions with a continuous outcome, while CST allows multiple actions with categorical outcomes. Moreover, X-learner is one-shot  while CST is trained in an iterative fashion.  

Self-training algorithms have been widely studied in semi-supervised learning and domain adaptation problems \citep{nigam2000text, amini2002semi, grandvalet2005semi,zou2019confidence,han2019unsupervised} in various forms. \citet{grandvalet2005semi} propose using entropy regularization for semi-supervised learning as a class-overlapping measure to utilize unlabeled data. \citet{nigam2000text, amini2002semi, zou2019confidence} formulate the pseudolabel imputation as classification EM algorithm and show its convergence under proper initialization. Recent theoretical analysis~\citep{wei2020theoretical} and empirical evidence shows input consistency loss such as VAT loss~\citep{miyato2018virtual} can further improve pseudolabeling in semi-supervised learning. \citet{han2019unsupervised} points out that pseudolabel imputation can be viewed as minimizing min-entropy as a type of R\'enyi entropy $\frac{1}{1-\alpha}\log(\sum_{i=1}^n p_i^\alpha)$ when $\alpha \rightarrow \infty$, and Shannon entropy in \cite{grandvalet2005semi} is the case when $\alpha \rightarrow 1$. Self-training is also shown to be effective in semi-supervised learning for many other machine learning models besides neural networks \citep{tanha2017semi,li2008self}.
Unlike traditional self-training where the target domain is given by the problem, we propose to construct a target domain by imputing pseudolabels on all unseen actions to simulate a pseudo-randomized trial. 

\section{Problem Statement}

Following the notation in~\citet{lopez2020cost}, we use $\mathcal{X}$ to represent an abstract space and $\mathbb{P}(x)$ is a probability distribution on $\mathcal{X}$. Each sample $x = x_1, \cdots, x_n \in \mathcal{X}^n$ is drawn independently from $\mathbb{P}(x)$. $\mathcal{A}$ is the \textit{discrete} action space that a central agent can select for each sample, after applying action $a$, a \textit{categorical} outcome $r_a \in \{1,\cdots, m\}$ is revealed to the agent. In precision medicine, $\mathcal{X}$ may represent a patient cohort,  $\mathcal{A}$ refers to feasible treatments for a disease, and $r$ can be the indicator of whether a patient recovers after the treatment. Similarly, $\mathcal{X},\mathcal{A},r$ can represent visitors, ads shown and whether the visitor clicks in online marketing. 


We focus on a pricing example to illustrate our method. We use $x \in \mathcal{X}^n \sim \mathbb{P}(x)$ to denote a customer. Let $\mathcal{A}$ represent finite price options a central agent can offer to customers. After offering price $a\in \mathcal{A}$, the agent observes the response from the customer $r_a \in \{0, 1\}$, \ie either a 1 (buy) or a 0 (no-buy). Our objective is to learn a function $f(x,a)$ by maximizing the log likelihood
\begin{align}
\mathbb{E}_{x\sim\mathbb{P}(x)}\sum_{a\in\mathcal{A}}\sum_{r} \mathbb{P}(r|x,a)\log{f(r|x,a)},    
\end{align}
where $\mathbb{P}(r|x,a)$ is the true probability of $r$ given $x,a$ 
and the historical data is generated by $\pi_0(a|x)$, a stochastic assignment policy~\citep{shalit2017estimating,lopez2020cost}. The estimation task is often referred to as demand  estimation~\citep{wales1983estimation}, which is critical for many downstream decisions such as  inventory optimization and revenue management~\citep{kok2007demand,mcgill1999revenue}. This is in contrast to CRM-based methods which use the final reward as its objective to learn a policy $\pi(a|x)$ that maximizes  $\mathbb{E}_{x\sim\mathbb{P}(x), p\sim\pi(a|x)} \mathcal{R}(r, x)$~\citep{swaminathan2015counterfactual}, where $\mathcal{R}(\cdot)$ is a pre-defined reward function. 

With observational data, the counterfactual outcome is not always identifiable. We use Rubin's potential outcome framework and assume consistency, overlap and ignorability which is a sufficient condition for identifying potential outcomes from historical data~\citep{imbens2009recent,pearl2017detecting}. Here we formally present the ignorability assumption~\citep{rubin2005causal, shalit2017estimating}:
\begin{assumption}[Ignorability]
Let $\mathcal{A}$ be action set, $x$ is context (feature), $r_a|x$ is observed outcome for action $a\in\mathcal{A}$ given context $x$, $r_a \perp\!\!\!\perp a | x, \forall a \in \mathcal{A}$.
\end{assumption}
In other words, we assume there is no unobserved confounders. This condition generally cannot be made purely based on data and requires some domain knowledge.

\section{Algorithm}
In this section, we introduce  the Counterfactual Self-Training (CST) algorithm, which can be viewed as  an extension of the direct method via domain adaptation. Unlike existing methods using representation learning, we propose a novel self-training style algorithm to account for the bias inherent in the observational data. 
The self-training algorithm works in an iterative fashion: First, after training a classifier $f(x,a)$ on a source dataset, pseudolabels are created by the best guess of $f$. Next, the model is trained on a target dataset, and the trained model is used to generate new pseudolabels. This idea is illustrated in Figure \ref{fig:st_ex}. 

To formulate the counterfactual learning problem as a  domain adaptation problem, 
the observational data are viewed as data sampled from a source distribution $\mathcal{D}_S = \mathbb{P}(x)\pi(a|x)$. The target domain is a controlled randomized trial on the same feature distribution to ensure a uniformly good approximation on all actions. Our goal is to transfer observational data from the source domain to a simulated pseudo-randomized trial via self-training. 
More specifically, CST tries to optimize the following objective:

\begin{align}
\label{eqn:obj}
    \underset{\theta, \hat{r}}{\min}~
    \mathcal{L}_\text{CST}=&\underbrace{\sum_{i=1}^N \Big(
    \sum_{k=1}^m - r_{i,a_i,k} \log f_{\theta}(r_{i,a_i,k}|x_i,a_i)
    }_{\mathcal{L}_{src}} + \nonumber\\
    &\sum_{a\in \mathcal{A}\setminus a_i}\sum_{k=1}^m- \hat{r}_{i,a,k} \log f_{\theta}(\hat{r}_{i,a,k}|x_i,a_i) \Big)\\ 
    & s.t. \quad \hat{r}_{i,a} \in \Delta^{(m-1)} \nonumber 
\end{align}

We use $\hat{r}$ to represent imputed pseudolabels, $r_{i,a,k}$ is the $k$-th entry of $r_{i,a}$, $\hat{r}_{i,a}$ is defined on a simplex and $r_{i,a}$ is one hot encoded. Note when counterfactual outcomes are available, this objective corresponds to the true negative log likelihood function for multi-label classification without partial feedback. The first term $\mathcal{L}_{src}$  in Equation~\ref{eqn:obj} corresponds to the loss used in direct method, defined over the factual data alone. Meanwhile, the second term refers to the loss defined over the imputed counterfactual data. In other words, in order to get a good model across all actions, we jointly optimize $\hat{r}, \theta$ which represent the true negative log likelihood when all counterfactuals are observed. 
Such objectives can be minimized by the EM algorithm via iteratively updating model and imputing pseudolabels (the pseudolabel is the minimum solution for $\hat{r}$ defined on the simplex~\citep{zou2019confidence}). 
To accomplish this, 
we first train an initial classifier $f_0(x,p)$ on observational data, then impute pseudolabels on all unseen actions from the observation data with  $\hat{r}_{i,a} = \Bar{f}_0(x_i,a)$, where the $k$-th entry is imputed by
$$
\Bar{f}(r_{i,a,k}|x_i,a) = 1 \text{ if } k = \argmax_c  f(r_{i,a,c}|x_i,a) \text{ else } 0.
$$
However, it has been shown that such a  procedure may suffer from local minima~\citep{grandvalet2005semi} or over-confident wrong pseudolabels~\citep{zou2019confidence}.~\citet{wei2020theoretical} shows when the underlying data distribution and pseudolabeler satisfies expansion assumption (See Definition 3.1 and Assumption 4.1, 3.3 in~\citet{wei2020theoretical}), self-training algorithms with input-consistency are able to achieve improvement from pseudolabeling (Theorem 4.3 in~\citet{wei2020theoretical}). Intuitively, the condition states that there needs to be many correct neighbors around the errors made by pseudolabeler so that correct labels can refine the decision boundary. 
Note that when imposing the same assumptions 4.1 and 3.3 on the underlying data distribution and error set of our pseudolabeler, we have the same improvement guarantee via self-training. We formalize this statement in Theorem~\ref{thm:main}, the formal definitions, assumptions and proof are included in Appendix due to the space limit. 

\begin{theorem}
\label{thm:main}
If the pseudolabeler has sufficiently good classification performance on all classes across actions, and the underlying data distribution satisfies expansion assumption with good separation (close samples have a high likelihood of having same labels), then minimizing a weighted sum of input consistency loss and pseudolabeling loss can improve pseudolabeler's classification performance. 
\end{theorem}

Motivated by the theoretical analysis that better input consistency can help further improve the pseudolabeling algorithm, we adapt Virtual Adversarial Training loss (VAT)~\citep{miyato2018virtual} in our CST algorithm, which we refer to as Counterfactual VAT (CVAT).
A key difference between CVAT and VAT is that CVAT considers the joint data-action space $(x, a)$ while VAT only regularizes on the input feature. 
Intuitively, VAT regularizes the output distribution to be isotropically smooth around each data point by selectively smoothing the model in its most anisotropic direction in an adversarial fashion. This direction is an adversarial perturbation $z^\text{adv}_i$ of input data $x_i$. Similar to VAT, CVAT minimizes the distance between $f_{\hat{\theta}}(x_i,a)$ and $f_\theta(x_i+z^\text{adv}_i,a)$ over all counterfactual actions for $x_i$,
\begin{align}
    &\mathcal{L}_\text{CVAT} = \sum_i \sum_{a \in \mathcal{A}\setminus a_i} {D\left[ f_{\hat{\theta}}(x_i,a), f_{\theta}(x_i + z^\text{adv}_i,a) \right]} \nonumber\\
    &z^\text{adv}_i = \argmax_{z; \Vert z \Vert_2 \le \epsilon}{ \sum_{a \in \mathcal{A}\setminus a_i} D\left[ f_{\hat{\theta}}(x_i,a), f_{\theta}(x_i + z, a) \right] }  \nonumber
\end{align}
\noindent where $D[\cdot,\cdot]$ is a divergence or distance measure between distributions and $\hat{\theta}$ stands for the current model parameters ($D(\cdot)$ denotes the actual function with substituted values). Like the original VAT, $D(z,x_i,a,\hat{\theta})$ takes the minimal value at $z=0$ for all $a$, the differentiability assumption dictates that the first derivative $\nabla_z \sum_{a \in \mathcal{A}\setminus a_i} D(z,x_i,a,\hat{\theta})|_{z=0}$ is zero. Therefore, for each $x_i$, the loss term can be approximated by its second-order Taylor approximation (around $z=0$) and $z^\text{adv}$ can still be effectively approximated by the first dominant eigenvector of the Hessian matrix via power iteration. 

The final CVAT regularized self-training objective can be written as
\begin{align}
\label{eqn:stvat}
    \mathcal{L}_\text{CST-CVAT} = \mathcal{L}_\text{CST} + \lambda \mathcal{L}_\text{CVAT}.
\end{align}
We iteratively train the model and impute pseudolabels with a fixed number of iterations.  The algorithm is stated in Algorithm~\ref{alg:cst}.
CST training with both loss $\mathcal{L}_\text{CST-CVAT}$ and $\mathcal{L}_\text{CST}$ is convergent, which we show in Proposition~\ref{prop:conv}. The proof is included in the Appendix. 

\begin{prop}
\label{prop:conv}
CST is convergent with a proper learning rate. 
\end{prop}

\begin{algorithm}[!htbp]
\caption{Counterfactual Self-Training}
\label{alg:cst}
\begin{algorithmic}[1]
    \State Select loss function  $\mathcal{L}_\text{ST} \leftarrow \mathcal{L}_\text{CST}$ in Equation~\ref{eqn:obj} or $\mathcal{L}_\text{ST} \leftarrow \mathcal{L}_\text{CST-CVAT}$ in Equation~\ref{eqn:stvat}.
    \State Train a base learner $f_0$ on observational data. 
    \For{each iteration}
    \For{each $a \in \mathcal{A} \setminus a_i$}
        \State Impute pseudolabel $\hat{r}_{i,a} = \Bar{f}_0(x_i,a)$. 
    \EndFor
    \While{not converged}
        \State $\theta \leftarrow \text{Optim}(\nabla_{\theta}\mathcal{L}_\text{ST}, \theta)$ \algorithmiccomment{Self-training}
        \For{each $i \in \{ 1 \ldots N \}$ and $a \in \mathcal{A} \setminus a_i$}
                \State $\hat{r}_{i,a} = \Bar{f}_0(x_i,a)$. \algorithmiccomment{Imputation}
        \EndFor
    \EndWhile
    \State $f_0 = f_\theta$
    \EndFor
\end{algorithmic}
\end{algorithm}

\section{Experiments}\label{sec:exp}
We begin with a toy example of binary actions  to visualize and illustrate the benefit brought by CST. Next we construct synthetic datasets for a pricing example and utilize
two real datasets to demonstrate the efficacy of our proposed algorithm. For the implementation, we use a three layer neural network with 128 nodes as our model with dropout of $p=0.2$ and Leaky ReLU activation unless specified otherwise, and cross entropy loss as the loss function. We adapt the following backbone models which are originally designed for counterfactual regression tasks to counterfactual classification in our experiments to warm start CST, the implementation details are included in the Appendix:
\begin{itemize}
    \item Direct Method (DM): This baseline directly trains a model on observational data.  
    \item HSIC~\citep{lopez2020cost}: We use the last layer as an embedding and calculate HSIC between the embedding and the actions, HSIC is then used as regularization in addition to cross entropy loss. 
    \item Uniform DM (UDM) ~\citep{Wang2019BatchLF}: This method estimates the logging policy using historical data and uses importance sampling to simulate a randomize trial. 
\end{itemize}

Once we finish training the backbone model, we run CST using pseudolabels generated by the backbone model according to Equation~\ref{eqn:obj} (PL) or 
with additional CVAT loss defined in Equation~\ref{eqn:stvat} (PL + CVAT). 
The distance measure used in CVAT loss is chosen to be KL divergence. The stepsize in finite-difference approximation, number of power iteration steps and perturbation size are set to 10, 3 and 1 respectively. The number of iteration of CST is set to 2. The hyperparameter of CVAT loss is chosen using a grid search of  $[0.01,0.1,1,10]$ on a validation set. 
All experiments are conducted using one NVidia GTX 1080-Ti GPU with five repetitions and Adam optimizer~\citep{kingma2014adam}. Mean and standard error are reported for each metric. For synthetic and real data experiments, we choose one downstream task to examine how each method performs by predicting which action is most likely to cause a positive outcome ($r=1$). 

\paragraph{Toy Data}

We present a toy example to illustrate CST procedure with PL + CVAT and offer some intuition on how this method works\footnote{Code is available at \url{https://github.com/ruijiang81/CST}}. In our examples, we assume there are two available offers/actions A0 and A1, and two types of customers P0 and P1. P0 reacts favorably  to offer A0 (i.e., resulted in purchases) and has no reaction to A1. P1, on the other hand, has exactly the opposite reactions towards the offers.  The underlying data distribution for P0 and P1 is a simple two moon distribution as shown in Figure~\ref{fig:toy} with shaded points. 
The training (factual) data points are generated through a biased logging policy which has a higher propensity $\exp(-(x_0-min(x_0))$ to assign action A1 for features with small $x_0$, which is the first dimension of feature $\mathcal{X}$. The direct method is chosen to be a three-layer MLP with 16 hidden nodes and dropout layer with $p = 0.5$.  

Figure~\ref{fig:dma0_ini},~\ref{fig:dma1_ini} show the data samples of 50 observed data points. The curve represents the decision boundary of the current model, initially trained with the direct method in Figure~\ref{fig:dma0_ini}-\ref{fig:impa1_ini}. It fails to converge to the optimal solution for action A1, despite  achieving a good performance for A0 on the training data. 
The CST algorithm imputes the pseudolabels and then repeats its training in an iterative fashion. Here we use $\lambda = 1$ for CVAT loss to enforce input consistency. Figure~\ref{fig:dma0_it0} and~\ref{fig:dma1_it0} show the output after the first imputation, while Figure~\ref{fig:dma0_it10} and~\ref{fig:dma1_it10} depict the result after 10 iterations. We observe that CST with PL + CVAT converges to the correct solution for both actions and the margin between the two classes becomes much larger with training. Compared to the direct method, our counterfactual self-training algorithm better utilizes the underlying data structure and refines the decision boundary through extrapolation across actions. 

\begin{figure*}[!htbp]%
    \centering
    \subfloat[\centering Toy Data Samples - A0\label{fig:dma0_ini}]{{\includegraphics[width=0.2\linewidth]{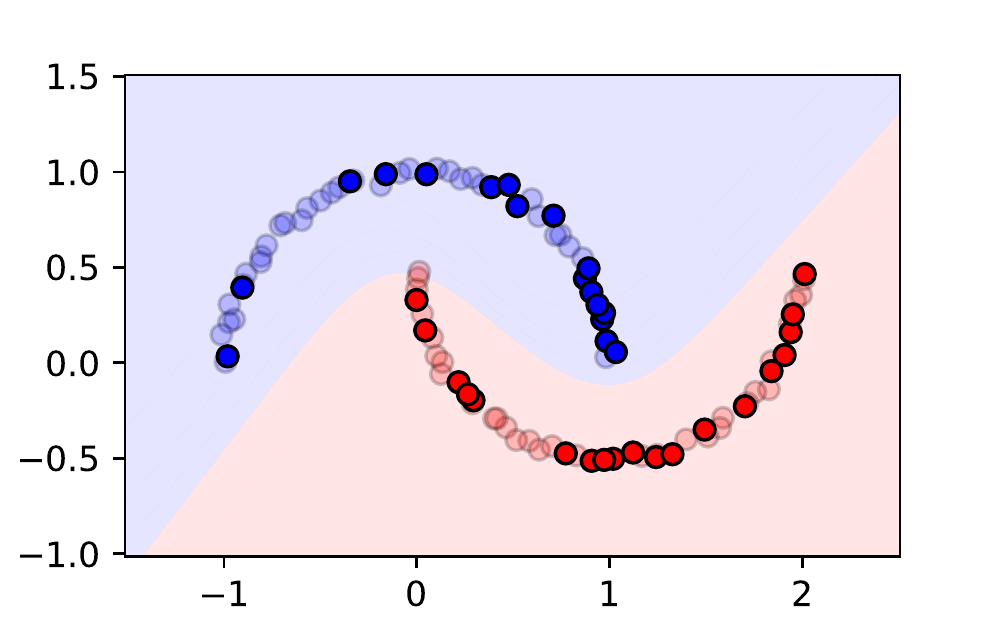} }}%
    \subfloat[\centering Toy Data Samples - A1\label{fig:dma1_ini}]{{\includegraphics[width=0.2\linewidth]{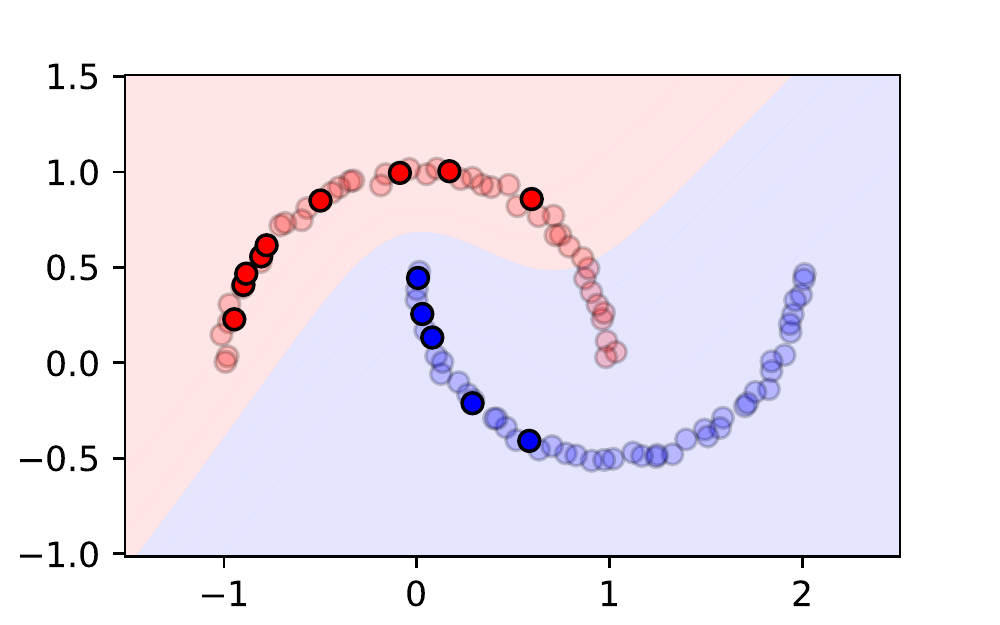} }}%
    \subfloat[\centering Imputed - A0\label{fig:impa1_ini}]{{\includegraphics[width=0.2\linewidth]{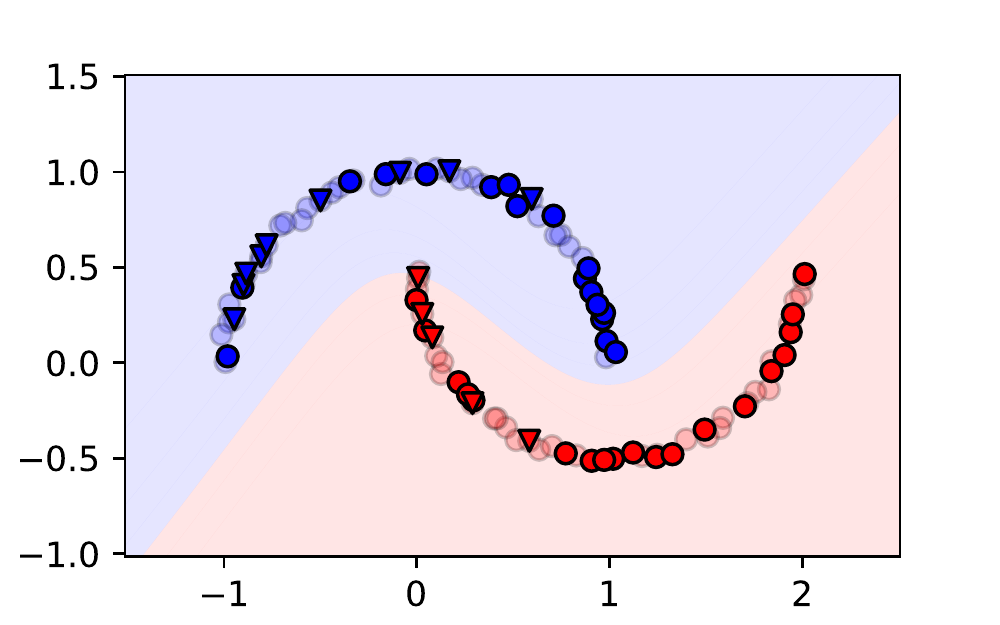} }}%
    \subfloat[\centering Imputed - A1\label{fig:impa1_ini}]{{\includegraphics[width=0.2\linewidth]{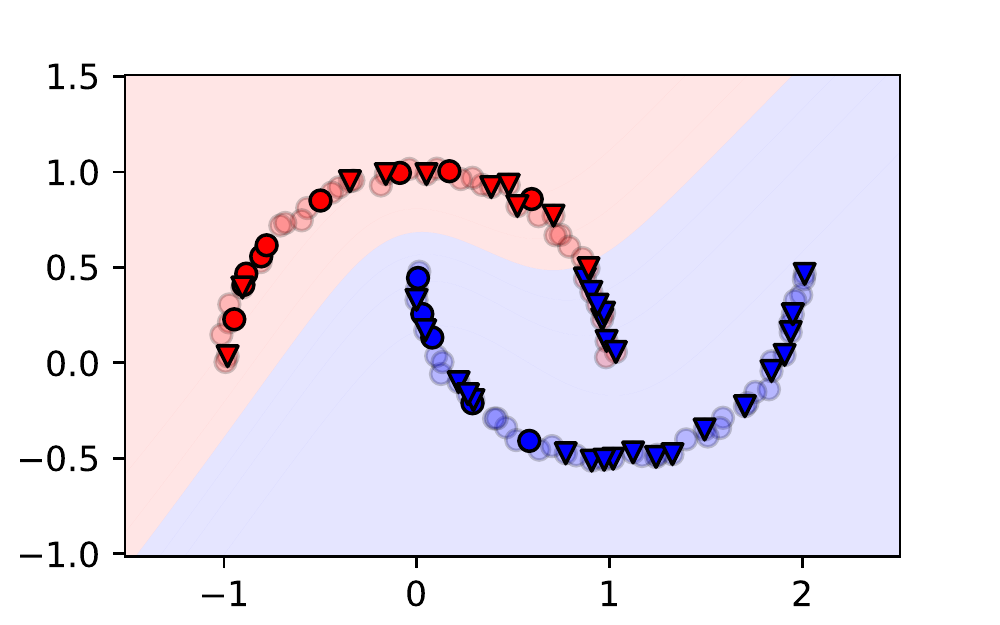} }} \\ %
    \subfloat[\centering PL Iter 1 - A0\label{fig:dma0_it0}]{{\includegraphics[width=0.2\linewidth]{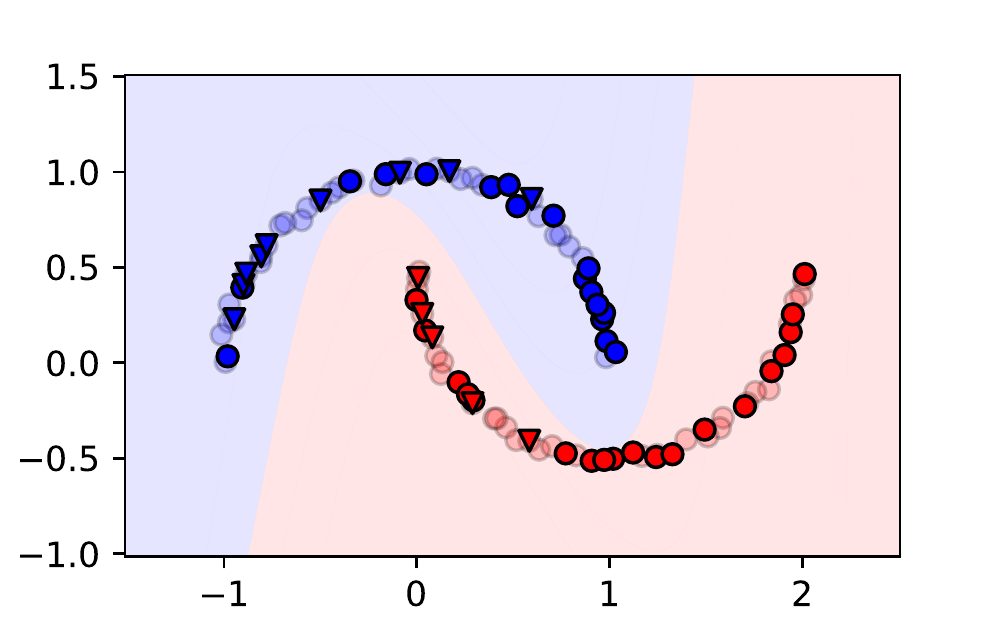} }}%
    \subfloat[\centering PL Iter 1 - A1\label{fig:dma1_it0}]{{\includegraphics[width=0.2\linewidth]{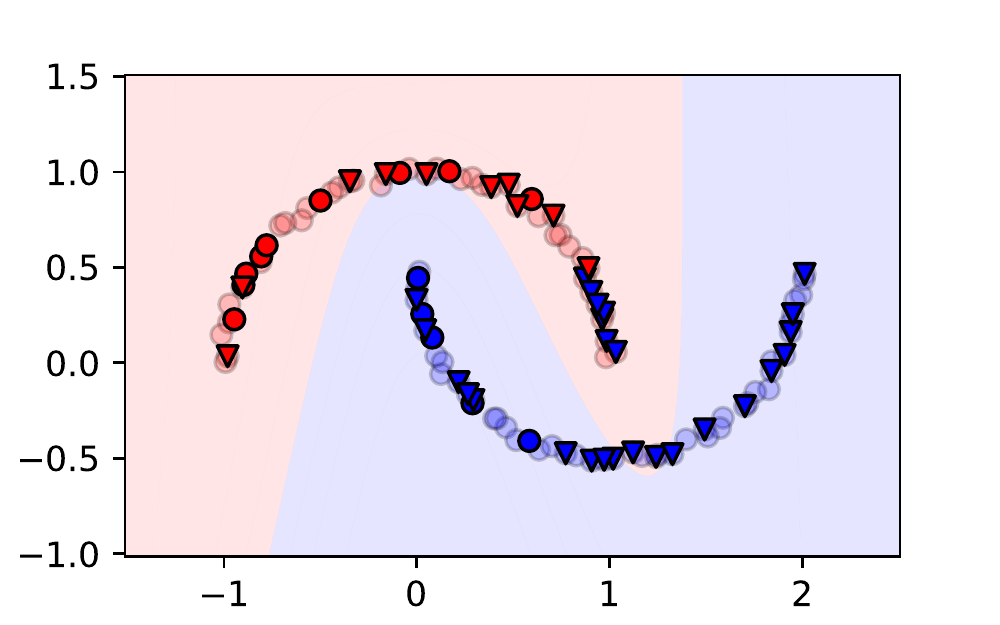} }}%
    \subfloat[\centering PL Iter 10 - A0\label{fig:dma0_it10}]{{\includegraphics[width=0.2\linewidth]{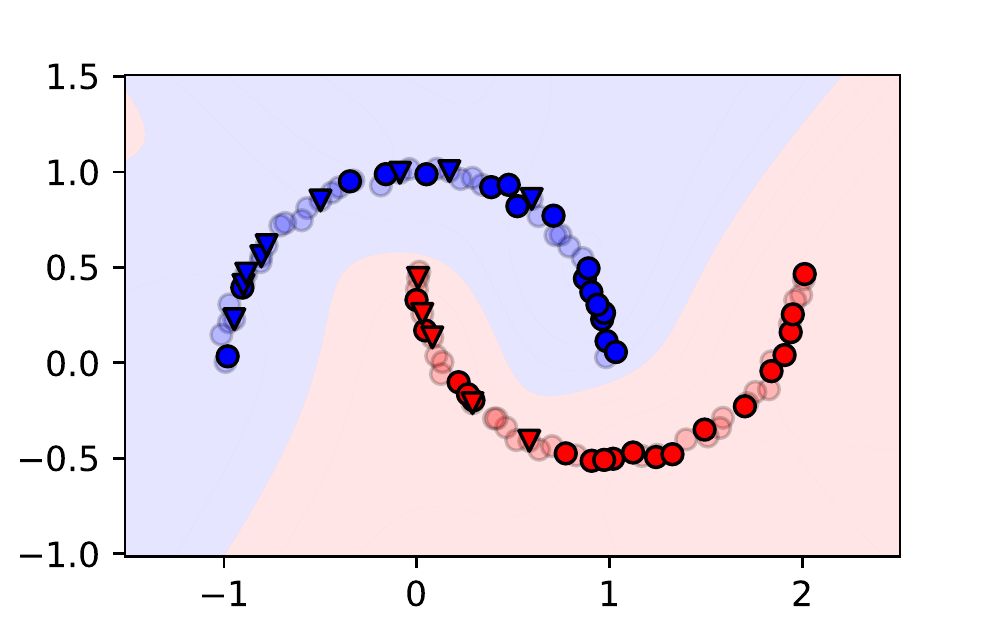} }}%
    \subfloat[\centering PL Iter 10 - A1\label{fig:dma1_it10}]{{\includegraphics[width=0.2\linewidth]{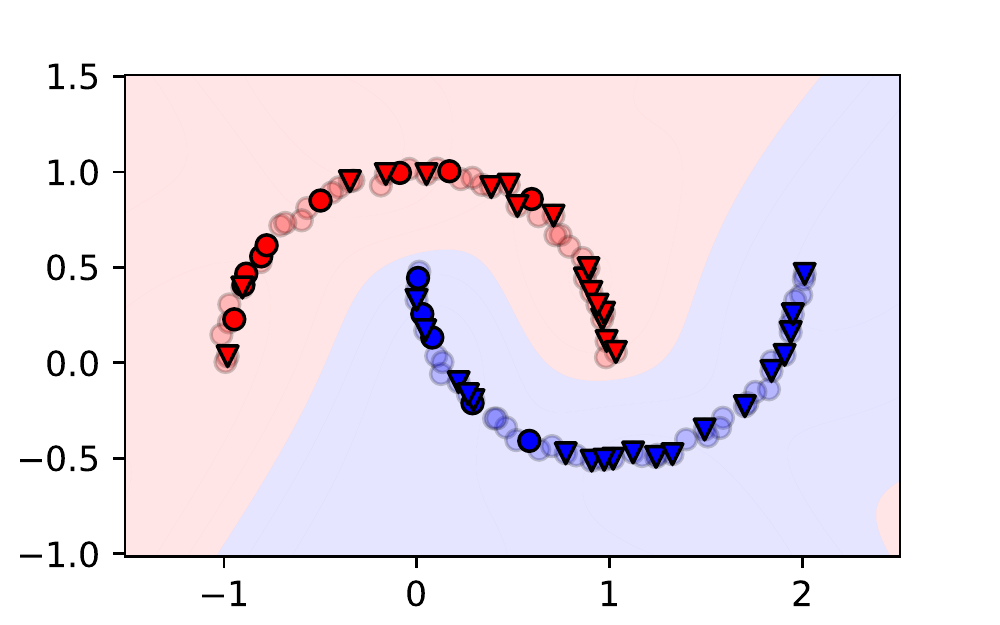} }} \\ %
    \caption{CST with PL + CVAT on Toy Data. Points and Triangles represent factual data and imputed pseudolabels respectively. The underlying data distribution is shown using shaded points. The curve represents the decision boundary of the current model. (c), (d): Direct Method Solution; (g), (h): CST with PL + CVAT Solution.}%
    \label{fig:toy}
\end{figure*}

\paragraph{Synthetic Datasets}

\begin{table*}[h]
    \centering
    \resizebox{0.85\textwidth}{!}{
    \begin{tabular}{lcccccc}\toprule 
      Model & Method & D1 & D2 & D3 & D4 & D5\\\hline
     \multirow{3}{*}{DM}&Backbone
     &2.3738$\pm$0.0747
     &2.9439$\pm$0.2395
     &3.2444$\pm$0.0856
     &3.3443$\pm$0.1210
     &2.4220$\pm$0.1275
     \\ 
     & PL 
     &0.5744$\pm$0.0165
     &0.6997$\pm$0.0317
     &0.7263$\pm$0.0205
     &0.8043$\pm$0.0348
     &0.6425$\pm$0.0253
     \\ 
     & PL + CVAT 
     &\textbf{0.5422}$\pm$0.0080
     &\textbf{0.5863}$\pm$0.0145
     &\textbf{0.6825}$\pm$0.0144
     &\textbf{0.6807}$\pm$0.0132
     &\textbf{0.5754}$\pm$0.0211
     \\ \hline 
    \multirow{3}{*}{HSIC}&Backbone
     &1.4748$\pm$0.0183
     &1.6062$\pm$0.0645
     &1.8629$\pm$0.0190
     &1.9334$\pm$0.0564
     &1.5240$\pm$0.0601
     \\ 
     & PL 
     &0.6028$\pm$0.0199
     &0.6840$\pm$0.0361
     &0.7177$\pm$0.0199
     &0.7423$\pm$0.0262
     &0.6220$\pm$0.0197
     \\ 
     & PL + CVAT 
     &\textbf{0.5418}$\pm$0.0147
     &\textbf{0.5759}$\pm$0.0074
     &\textbf{0.6553}$\pm$0.0086
     &\textbf{0.6593}$\pm$0.0108
     &\textbf{0.5591}$\pm$0.0135
     \\ \hline 
    \multirow{3}{*}{UDM}&Backbone
     &2.5589$\pm$0.2602
     &2.7615$\pm$0.6944
     &1.6277$\pm$0.4775
     &3.0926$\pm$0.5732
     &2.5924$\pm$0.6082
     \\ 
     & PL 
     &0.5841$\pm$0.0161
     &0.6315$\pm$0.0229
     &0.7090$\pm$0.0195
     &0.7258$\pm$0.0118
     &0.5936$\pm$0.0137
     \\ 
     & PL + CVAT 
     &\textbf{0.5553}$\pm$0.0238
     &\textbf{0.5891}$\pm$0.0221
     &\textbf{0.6502}$\pm$0.0087
     &\textbf{0.6446}$\pm$0.0039
     &\textbf{0.5616}$\pm$0.0145
     \\ \bottomrule
    \end{tabular}
    }
    \caption{Log Loss on Synthetic Dataset.}
    \label{tab:syn_soft}
\end{table*}

\begin{table*}[!htbp]
    \centering
    \resizebox{0.85\textwidth}{!}{
    \begin{tabular}{lcccccc}\toprule
      Model  & Method & D1 & D2 & D3 & D4 & D5\\\hline
     \multirow{3}{*}{DM}&Backbone
     &0.2922$\pm$0.0100
     &0.2775$\pm$0.0030
     &0.3588$\pm$0.0046
     &0.3560$\pm$0.0056
     &0.2816$\pm$0.0090
     \\ 
     & PL 
     &\textbf{0.2561}$\pm$0.0057
     &0.2564$\pm$0.0046
     &0.3363$\pm$0.0066
     &0.3244$\pm$0.0039
     &0.2545$\pm$0.0060
     \\ 
     & PL + CVAT 
     &0.2566$\pm$0.0039
     &\textbf{0.2544}$\pm$0.0029
     &\textbf{0.3299}$\pm$0.0050
     &\textbf{0.3219}$\pm$0.0037
     &\textbf{0.2494}$\pm$0.0047
     \\ \hline 
    \multirow{3}{*}{HSIC}&Backbone
     &0.2804$\pm$0.0093
     &0.2767$\pm$0.0040
     &0.3573$\pm$0.0054
     &0.3551$\pm$0.0083
     &0.2785$\pm$0.0076
     \\ 
     & PL 
     &\textbf{0.2623}$\pm$0.0097
     &0.2554$\pm$0.0030
     &0.3328$\pm$0.0059
     &0.3308$\pm$0.0097
     &0.2554$\pm$0.0053
     \\ 
     & PL + CVAT 
     &0.2673$\pm$0.0082
     &\textbf{0.2553}$\pm$0.0030
     &\textbf{0.3285}$\pm$0.0045
     &\textbf{0.3214}$\pm$0.0040
     &\textbf{0.2544}$\pm$0.0062
     \\ \hline 
    \multirow{3}{*}{UDM}&Backbone
     &0.2718$\pm$0.0077
     &0.2710$\pm$0.0023
     &0.3645$\pm$0.0043
     &0.3639$\pm$0.0085
     &0.2799$\pm$0.0055
     \\ 
     & PL 
     &\textbf{0.2602}$\pm$0.0086
     &0.2603$\pm$0.0029
     &\textbf{0.3390}$\pm$0.0075
     &0.3380$\pm$0.0090
     &0.2620$\pm$0.0044
     \\ 
     & PL + CVAT 
     &0.2608$\pm$0.0082
     &\textbf{0.2580}$\pm$0.0027
     &0.3496$\pm$0.0144
     &\textbf{0.3274}$\pm$0.0019
     &\textbf{0.2564}$\pm$0.0029
     \\ \bottomrule 
    \end{tabular}
    }
    \caption{Hamming Loss on Synthetic Dataset.}
    \label{tab:syn_hamm}
\end{table*}

\begin{table*}[h]
    \centering
    \begin{tabular}{lcccc}\toprule 
       Model + Metric & Method & $o=1$&$o=2$&$o=3$  \\\hline
     \multirow{3}{*}{DM + Hamming Loss}& Backbone
     &0.2886$\pm$0.0065
     &0.2850$\pm$0.0085
     &0.2901$\pm$0.0093
     \\ 
     & PL 
     &0.2611$\pm$0.0045
     &\textbf{0.2645}$\pm$0.0097
     &\textbf{0.2606}$\pm$0.0048
     \\ 
     & PL + CVAT 
     &\textbf{0.2605}$\pm$0.0059
     &0.2677$\pm$0.0145
     &0.2696$\pm$0.0092
     \\ \hline 
     \multirow{3}{*}{DM + Log Loss}& Backbone
     &2.4871$\pm$0.1057
     &2.4993$\pm$0.1100
     &2.3416$\pm$0.0093
     \\ 
     & PL 
     &0.6200$\pm$0.0316
     &0.6053$\pm$0.0237
     &0.6132$\pm$0.0213
     \\ 
     & PL + CVAT 
     &\textbf{0.5643}$\pm$0.0118
     &\textbf{0.5672}$\pm$0.0170
     &\textbf{0.5664}$\pm$0.0126
     \\ \bottomrule 
    \end{tabular}
    \caption{Experiments with Varying Overlap}
    \label{tab:varyovp}
\end{table*}

In synthetic experiments, we use a pricing example similar to the experiment in~\citet{lopez2020cost}. Let $U(\cdot,\cdot)$ be a uniform distribution. Assume customer features are a 50-dimensional vector $X$ drawn from $U(0,1)^{50}$ and there are 5 price options from \$1 to \$5. The logging policy is set as $\pi(p=i|x) = \frac{x_i}{\sum_{i=1}^{10}x_i}$. Five types of demand functions are simulated, and the complete data generation process is detailed in the Appendix. For each action, there is a binary outcome: purchase (1) or not purchase (0).

We generate 1000 samples for each demand function and report the true negative log likelihood (NLL) to examine probability estimation of each method in Table \ref{tab:syn_soft}. Hamming loss which relies on the hard labels generated by the algorithm is reported in Table \ref{tab:syn_hamm}.
Since we want to fit a classifier that may work for many possible downstream tasks, the main metric we focus is NLL given current classifiers. For example, in revenue management, estimating demand accurately is critical for price optimization, inventory control and promotion planning. More specifically, NLL is defined as 
\begin{align}
    \sum_{i=1}^N \Big(
    &\sum_{k=1}^m - r_{i,a_i,k} \log f_{\theta}(r_{i,a_i,k}|x_i,a_i) + \\ \nonumber 
    &\sum_{a\in \mathcal{A}\setminus a_i}\sum_{k=1}^m- {r}_{i,a,k} \log f_{\theta}({r}_{i,a,k}|x_i,a_i) \Big)/N|\mathcal{A}|
\end{align}

and hamming loss is defined as $\sum_{i=1}^N \sum_{a\in\mathcal{A}}$
$\mathbb{I}(y_{i,a}=\hat{y}_{i,a})/N|\mathcal{A}|$. 

With respect to different backbone models, we find UDM and HSIC in general improve over Direct Method with a better NLL, with HSIC being the best model. In all experiments, we find PL improves NLL from backbone model and incorporating CVAT loss further improves the performance, highlighting the value of our proposed CST algorithm. 
Similarly, pseudolabeling and CVAT loss can provide improvement for prediction accuracy measured by hamming loss. In some experiments, PL achieves the best hamming loss, this discrepancy between NLL and hamming loss may be caused by over-confident wrong predictions made by pseudolabels~\citep{zou2019confidence} and including CVAT loss provides a viable way to smooth output predictions to avoid over-confident predictions.


\paragraph{Multi-Label Datasets}

\begin{table}[!htbp]
    \centering
    \resizebox{\linewidth}{!}{
    \begin{tabular}{lcccc}\toprule 
       Model & Method & Scene & Yeast & \\\hline
     \multirow{3}{*}{DM}&Backbone
     &3.2289$\pm$0.1678
     &5.5001$\pm$0.1218
     &
     \\ 
     & PL 
     &0.3357$\pm$0.0132
     &0.6956$\pm$0.0318
     &
     \\ 
     & PL + CVAT 
     &\textbf{0.3272}$\pm$0.0124
     &\textbf{0.5234}$\pm$0.0053
     &
     \\ \hline 
    \multirow{3}{*}{HSIC}&Backbone
     &1.0385$\pm$0.0320
     &1.9315$\pm$0.0681
     &
     \\ 
     & PL 
     &0.3667$\pm$0.0061
     &0.7372$\pm$0.0980
     &
     \\ 
     & PL + CVAT 
     &\textbf{0.3503}$\pm$0.0056
     &\textbf{0.5199}$\pm$0.0048
     &
     \\ \hline 
    \multirow{3}{*}{UDM}&Backbone
     &6.2181$\pm$0.4192
     &7.5073$\pm$0.9486
     &
     \\ 
     & PL 
     &0.5526$\pm$0.0160
     &0.7785$\pm$0.0924
     &
     \\ 
     & PL + CVAT 
     &\textbf{0.4373}$\pm$0.0152
     &\textbf{0.5441}$\pm$0.0119
     &
     \\ \bottomrule 
    \end{tabular}
    }
    \caption{Log Loss on Real Dataset.}
    \label{tab:real_soft}
\end{table}

\begin{table}[!htbp]
    \centering
    \resizebox{\linewidth}{!}{
    \begin{tabular}{lcccc}\toprule 
       Model & Method & Scene & Yeast & \\\hline
     \multirow{3}{*}{DM}&Backbone
     &0.1301$\pm$0.0036
     &0.2853$\pm$0.0080
     &
     \\ 
     & PL 
     &0.1304$\pm$0.0034
     &0.2539$\pm$0.0066
     &
     \\ 
     & PL + CVAT 
     &\textbf{0.1258}$\pm$0.0024
     &\textbf{0.2434}$\pm$0.0054
     &
     \\ \hline 
    \multirow{3}{*}{HSIC}&Backbone
     &0.1821$\pm$0.0031
     &0.2821$\pm$0.0073
     &
     \\ 
     & PL 
     &0.1546$\pm$0.0041
     &0.2471$\pm$0.0052
     &
     \\ 
     & PL + CVAT 
     &\textbf{0.1473}$\pm$0.0033
     &\textbf{0.2366}$\pm$0.0043
     &
     \\ \hline 
    \multirow{3}{*}{UDM}&Backbone
     &0.2542$\pm$0.0079
     &0.3252$\pm$0.0040
     &
     \\ 
     & PL 
     &0.2274$\pm$0.0129
     &0.2910$\pm$0.0100
     &
     \\ 
     & PL + CVAT 
     &\textbf{0.1852}$\pm$0.0081
     &\textbf{0.2868}$\pm$0.0119
     &
     \\ \bottomrule 
    \end{tabular}
    }
    \caption{Hamming Loss on Real Dataset.}
    \label{tab:real_hamm}
\end{table}

We use two multi-label datasets from  \href{https://www.csie.ntu.edu.tw/~cjlin/libsvmtools/datasets/multilabel.html}{LIBSVM repository}~\citep{elisseeff2002kernel,boutell2004learning,chang2011libsvm}, which are used for semantic scene and text classification. We convert the supervised learning datasets to bandit feedback by creating a logging policy using 5\% of the data following ~\citet{swaminathan2015counterfactual,lopez2020cost}. More specifically, each feature $x$ has a label $y\in \{0,1\}^p$ where $p$ is the number of labels. After the logging policy selects a label (action) $i$, a reward $y_i$ is revealed as bandit feedback $(x,i,y_i)$, \ie  for each data point, if the policy selects one of the correct labels of that data point, it gets a reward of 1 and  0 otherwise. By doing so, we have the full knowledge of counterfactual outcomes for evaluation.  Data statistics are summarized in the Appendix. 
True negative log loss and hamming loss are reported in Table \ref{tab:real_soft} and \ref{tab:real_hamm} respectively.

On both datasets, we observe that pseudolabeling and CVAT loss improve backbone model performance in terms of hamming loss and negative log loss and CVAT loss brings additional improvement. The improvement of pseudolabeling and CVAT loss is also dependent on backbone model performance.  In particular, we observe a better base model leads to better CST  performance.  


\paragraph{Varying Overlap}
In this section, we vary the level of overlap to control the randomness of observational data. We use the synthetic data setup using D1 and set logging policy as $\pi(p=i|x) = \frac{e^{o*x_i}}{\sum_{i=1}^{10}e^{o*x_i}}$. The results are summarized in Table~\ref{tab:varyovp}. We find both DM and our proposed method is robust to different level of randomness in observational data. Similar conclusion is also observed in~\citet{sachdeva2020off}.

\section{Conclusion and Future Work}
In this paper, we propose a novel counterfactual self-training algorithm for counterfactual classification problem. CST relies on pseudolabeling to estimate the true log likelihood on observational data and works in an iterative fashion. With expansion and separation assumptions, CST experiences the same theoretical guarantees for performance improvement as shown in existing pseudolabeling literature. Our results hold promises for many counterfactual classification models that may work beyond neural network models due to the model agnostic nature of pseudolabeling. However, our CST framework has several limitations. First, CST requires finite \textit{categorical} action set, which is also a crucial requirement for CRM methods. In order to augment observation data, CST will augment every action not observed. For continuous action, discretization or joint kernel embedding proposed in~\cite{zenati2020counterfactual} might be used as an extension to CST, which we leave for future work. Second, while our experimental results show encouraging results that CST can improve even state-of-the-art models, the expansion and base-model performance assumption required in current theoretical understanding generally cannot be verified, which means CST may deteriorate model performance when the underlying data is noisy or base model has a high error rate. In practice, practitioners can use model selection methods such as counterfactual cross validation~\citep{saito2020counterfactual} to avoid this, and we argue CST still provides a promising means for enhancing counterfactual classification performance. 


\newpage
\section{Acknowledgement}
We thank the useful suggestions from anonymous reviewers.  
\bibliography{arxiv}

\onecolumn
\begin{center}
\Large \textbf{Appendix for Enhancing Counterfactual Classification via Self-Training}    
\end{center}

\normalsize

\subsection{Proof of Proposition 1.}
\begin{prop}
CST is convergent with a proper learning rate. 
\end{prop}

\begin{proof}
For $\lambda \geq 0$, our CST objective with loss $\mathcal{L}_\text{CST}$ or $\mathcal{L}_\text{CST-CVAT}$ can be written as 
\begin{align}
\label{eqn:propobj}
    \underset{\theta, \hat{r}}{\min} \mathcal{L}_\text{CST-CVAT}=\sum_{i=1}^N \Big(\sum_{k=1}^m - r_{i,k} \log f_{\theta}(r_{i,k}&|x_i,a_i) + \sum_{a\in \mathcal{A}\setminus a_i}\sum_{k=1}^m- \hat{r}_{i,a,k} \log f_{\theta}(\hat{r}_{i,a,k}|x_i,a) \Big) + \lambda \mathcal{L}_\text{CVAT}\\ 
    & s.t. \quad \hat{r}_{i,a} \in \Delta^{(m-1)} \nonumber 
\end{align}
where $r_i$ is the factual data observed and $\hat{r}_{i,a}$ is imputed by a trained classifier $f_\theta$.  
We show the convergence of CST, which imputes pseudolabels using the argmax operation.
The objective is optimized via the following two steps:

\noindent \textbf{1) Pseudolabel Imputation:} Fix $\theta$ and impute $\hat{r}$ to solve:

\begin{align}
    \label{eqn:proppl}
    \min_{\hat{r}} \sum_{i=1}^N \sum_{a\in \mathcal{A}\setminus a_i}\sum_{k=1}^m 
    &- \hat{r}_{i,a,k} \log f_{\theta}(\hat{r}_{i,a,k}|x_i,a) \\ 
    & s.t. \quad \hat{r}_{i,a} \in \Delta^{(m-1)}, \forall i, a \nonumber
\end{align}

\noindent where $\Delta^{(m-1)}$ is a simplex defined on the outcome space, and $r_{i,a,k}$ is the $k$-th entry of $r_{i,a}$.

\noindent \textbf{2) Model Retraining:} Fix $\hat{r}$ and solve the following optimization using gradient descent:

\begin{align}
\label{eqn:proploss}
    \underset{\theta}{\min}
    &{\sum_{i=1}^N \Big(
    \sum_{k=1}^m - r_{i,k} \log f_{\theta}(r_{i,k}|x_i,a_i)
    } + 
    \sum_{a\in \mathcal{A}\setminus a_i}\sum_{k=1}^m- \hat{r}_{i,a,k} \log f_{\theta}(\hat{r}_{i,a,k}|x_i,a) \Big) + \lambda \mathcal{L}_\text{CVAT}
\end{align}

For CST, we have:

\noindent \textbf{Step 1) is non-increasing:}
(\ref{eqn:proppl}) has an optimal solution which is given by pseudolabels imputed by argmax operation with feasible set being all possible outcomes. As a result, (\ref{eqn:proppl}) is non-increasing. 

\noindent \textbf{Step 2) is non-increasing:} 
If one use gradient descent to minimize the loss defined in Equation~\ref{eqn:proploss}. The
loss is guaranteed to decrease monotonically with a proper learning rate. For mini-batch gradient descent commonly used in practice, the loss is not guaranteed to decrease but also almost certainly converge to a lower value~\citep{zou2019confidence}. 

Since loss in Equation~\ref{eqn:propobj} is lower bounded, CST is convergent.
\end{proof}

\subsection{Theoretical Guarantee of CST}
We restate the informal statement of Theorem~\ref{thm:main} in main paper. 

\begin{mythm}{1}[Informal]
If the pseudolabeler has sufficiently good classification performance on all classes across actions, and the underlying data distribution satisfies expansion assumption with good separation (close samples have a high likelihood of having same labels), then minimizing a weighted sum of input consistency loss and pseudolabeling loss can improve pseudolabeler's classification performance.
\end{mythm}

First, we present the expansion assumption, which states for data distribution $P$, each small set $V$ in the conditional distribution each class $i$ $P_i$, has a sufficiently large neighborhood. 

\begin{definition}[$(\alpha, c)$-Expansion (Definition 3.1 in~\citet{wei2020theoretical})] Define neighborhood of $x$ as $\mathcal{N}(x) = \{x': \mathcal{B}(x) \cap \mathcal{B}(x')\neq 0\}$, where 
\begin{align}
 \mathcal{B}(x)\overset{\Delta}{=}\{x':\exists T\in\mathcal{T} \text{ s.t. } \|x'-T(x)\|\leq \epsilon\}   
\end{align}
is a small neighborhood of $x$ under some set of data transformations $\mathcal{T}$. $\mathcal{T}$ can be some data augmentation which is domain-specific. For all $ V \subseteq \mathcal{X}$ with $P_i (V) \leq \alpha$, the following holds:
\begin{align}
    P_i(\mathcal{N}(V)) \geq \min\{cP_i(V), 1\}
\end{align}
If $\forall i \in [m]$, $P_i$ satisfies $(\alpha, c)$ expansion, then $P$ satisfies $(\alpha, c)$ expansion.
\end{definition}

With the expansion assumption, next we formally present Theorem~\ref{thm:main}. 

\begin{theorem}[Formal]
\label{thm:main}
For any action $a \in \mathcal{A}$, suppose the maximum fraction of examples in any class which are mistakenly pseudolabeled by pseudolabeler $f^\text{pl}$ is $\alpha < 0.2$. Define $\mathcal{T}'$ such that $\forall T' \in \mathcal{T}', T'(\mathcal{X}\times\mathcal{A}) = T(\mathcal{X})\times\mathcal{A}$, under $\mathcal{T}'$, $P$ satisfies $(\alpha,c)$ expansion assumption with $c>5$. 
The underlying data distribution satisfies the separation assumption with
\begin{align}
    \mathcal{R}_\mathcal{B}(f^\star)\overset{\Delta}{=}\mathbb{E}_\mathbb{P}[\mathbf{1}(\exists (x')\in\mathcal{B}(x) \text{ s.t. } f^\star(x',a)\neq f^\star(x,a))]) \leq \mu, \forall a \in \mathcal{A}
\end{align}
\noindent where $f^\star$ is ground-truth classifier. Define $\text{Err}(f)$ is the 0-1 loss of $f$ with $f^*$, then for any minimizer $f$ of loss,
\begin{align}
    L(f) =&\frac{2c}{c-1}\mathcal{R}_\mathcal{B}(f)+\frac{2c}{c-1}L_{0-1}(f,f^\text{pl})-Err(f^\text{pl}), 
\end{align}
\noindent we have
\begin{align}
    \text{Err}(f) \leq \frac{2}{c-1}\text{Err}(f^\text{pl}) + \frac{2c}{c-1}\mu 
\end{align}
\noindent let $f_i$ be the classifier learnt in $i$-th iteration, 
\begin{align}
\label{eqn:iterative}
    Err(f_i) \rightarrow \frac{2c}{c-3}\mu, \text{ as }i \rightarrow \infty.
\end{align}
\end{theorem}

\begin{proof}
Define $\mathcal{X}' = \mathcal{X} \times \mathcal{A}$, first under $\mathcal{T}'$, the expansion assumption holds. Since for every action $a\in\mathcal{A}$, the maximum fraction of mistakes in any class $\alpha < 0.2$, for $P$, it is easy to verify the maximum fraction of mistakes in any class for $P_i$ $\Bar{\alpha}<0.2$. Last, the separation assumption is automatically met given the definition of $\mathcal{T}'$. The proof is completed by directly applying Theorem 4.3 in~\citet{wei2020theoretical}. Here $L(f)$ is a weighted sum of input consistency loss and the loss for fitting pseudolabels, which offers the theoretical guarantee for CST. In CST, the source loss is added for training stability, which is often used in self-training algorithms~\citep{zou2019confidence,wei2020theoretical,rizve2021defense}.
For Equation~\ref{eqn:iterative}, it is a direct result of sum of geometric sequence. 
\end{proof}

\subsection{Implementation of Backbone Models}
Since HSIC~\citep{lopez2020cost} and Uniform DM (UDM) is originally designed for continuous outcome variable. For a fair comparison, we adapt them into our Counterfactual Classification task. For HSIC, we simply replace the source loss function with entropy loss, with the final loss function as:

\begin{align}
    \mathcal{L}_\text{HSIC}=&{\frac{1}{N}\sum_{i=1}^N 
    \sum_{k=1}^m - r_{i,k} \log f_{\theta}(r_{i,k}|x_i,a_i) 
    } + \lambda \text{HSIC}(z_i, a_i)
\end{align}

\noindent where $z_i$ is the embedding of $x_i$ by $z_i = g_\phi(x_i)$ which we choose to be the output of second linear layer in our neural network, HSIC is defined as 

\begin{align}
    \text{HSIC}_n(z_i, a_i) = \frac{1}{N^2}\sum_{i,j}^N k(a_i,a_j)l(z_i,z_j) + \frac{1}{N^4}\sum_{i,j,k,l}^N k(a_i,a_j)l(z_k,z_l) - \frac{2}{N^3}\sum_{i,j,k}^N k(a_i,a_j)l(z_i,z_k)
\end{align}

\noindent where $k(\cdot, \cdot), l(\cdot, \cdot)$ are kernel functions which we choose to be RBF-Kernel with $\sigma = 0.5$. The action is represented with one-hot encoded vectors. In all experiments, $\lambda$ is chosen to be $0.01$, which experimentally had the best performance. 

For UDM, we minimize the following objective:

\begin{align}
    \mathcal{L}_\text{UDM}=&{\frac{1}{N}\sum_{i=1}^N 
    \frac{1}{\hat{\pi}_0(a_i|x_i)}\sum_{k=1}^m - r_{i,k} \log f_{\theta}(r_{i,k}|x_i,a_i) 
    } 
\end{align}

\noindent $\hat{\pi}_0(a|x)$ is estimated from observational data using a logistic regression model in our experiments.

\subsection{Data Generation for Synthetic Dataset}\label{sec:data}
In the synthetic experiments, we use a pricing example similar to the experiment in~\citet{lopez2020cost}. Let $U(\cdot,\cdot)$ be a uniform distribution. Assume customer features are a 50-dimensional vector $X$ drawn from $U(0,1)^{50}$ and there are 5 price options from \$1 to \$5. The logging policy is set as $\pi(p=i|x) = \frac{x_i}{\sum_{i=1}^{10}x_i}$. $\sigma$ denotes the sigmoid function. We simulated five types of demand functions, with $h(x) = \sum a_i \sum \exp(\sum b_j\|x_j-c_j\|)$, $a, b, c ~\sim U(0,1)^{50}$, $r \in \{0, 1\}$ :
\begin{itemize}
    \item $r \sim \sigma(h(x) - 2 x_0 \cdot p)$
    \item $r \sim \sigma(5 \cdot (x_0-0.5) - 0.4 \cdot p)$
    \item $r \sim \sigma(h(x) - \texttt{stepwise1}(x_0) \cdot p)$
    \item $r \sim \sigma(h(x) - \texttt{stepwise2}(x_0, x_1) \cdot p)$
    \item $r \sim \sigma(h(x) - (x_0+x_1) \cdot p)$
\end{itemize}

\noindent where the stepwise functions are defined as:
  \begin{equation}
    \texttt{stepwise1}(x)=
    \begin{cases}
      0.7, & \text{if}\ x\leq0.1 \\
      0.5, & \text{if}\ 0.1<x\leq0.3 \\
      0.3, & \text{if}\ 0.3<x\leq 0.6 \\
      0.1, & \text{if}\ 0.6<x\leq 1 \\
    \end{cases}
  \end{equation}
  
  \begin{equation}
    \texttt{stepwise2}(x,y)=
    \begin{cases}
      0.65, & \text{if}\ x\leq0.1 \text{ and } y>0.5\\
      0.45, & \text{if}\ x\leq0.1 \text{ and } y\leq0.5\\
      0.55, & \text{if}\ 0.1<x\leq0.3 \text{ and } y>0.5\\
      0.35, & \text{if}\ 0.1<x\leq0.1 \text{ and } y\leq0.5\\
      0.45, & \text{if}\ 0.3<x\leq0.6 \text{ and } y>0.5\\
      0.25, & \text{if}\ 0.3<x\leq0.6 \text{ and } y\leq0.5\\
      0.35, & \text{if}\ 0.6<x\leq1 \text{ and } y>0.5\\
      0.15, & \text{if}\ 0.6<x\leq1 \text{ and } y\leq0.5\\
    \end{cases}
  \end{equation} 

\subsection{Multi-Label Datasets Statistics}~\label{sec:realdata}
The statistics of multi-label datasets used in experiments are reported in Table~\ref{tab:real_stat}.
\begin{table}[h]
    \centering
    \begin{tabular}{lcccc}\toprule
         & \# Features & \# Labels&Train Size& Test Size\\\hline
     Yeast&103&14&1208&709\\  
     Scene&294&6&1203&704\\  \bottomrule
    \end{tabular}
    \caption{Dataset Statistics.}
    \label{tab:real_stat}
\end{table}

\end{document}